\documentclass[10pt,twocolumn,letterpaper]{article}

\usepackage{wacv}
\usepackage{times}
\usepackage{epsfig}
\usepackage{graphicx}
\usepackage{amsmath}
\usepackage{amssymb}

\usepackage{amsthm}
\DeclareMathAlphabet{\mathcalligra}{T1}{calligra}{m}{n}
\DeclareMathAlphabet{\mathantt}{OT1}{antt}{li}{it}
\DeclareMathAlphabet{\mathpzc}{OT1}{pzc}{m}{it}

\newcommand{\argmax}{\mathop{\rm argmax}}

\newcommand{\argmin}{\mathop{\rm argmin}}

\newcommand{\D}{\mathcal{D}}

\def\G{\mathcal{G}}

\newcommand{\V}{\mathcal{V}}

\newcommand{\W}{\mathcal{W}}

\newcommand{\E}{\mathcal{E}}

\newcommand{\Y}{\mathcal{Y}}

\newcounter{myRomanCounter}

\renewcommand*{\paragraph}[1]{\par{\normalsize\bf #1}\ }

\def\epsilon{\varepsilon}

\usepackage{aliascnt}
\newlength{\myskip}
  {\begin{list}{\arabic{enumi}.}%
     {\topsep=0in\itemsep=0in\parsep=0pt\partopsep=0in\usecounter{enumi}}%
   }{\end{list}}

\renewenvironment{itemize}%
  {\begin{list}{$\bullet$}%
     {\topsep=0in\itemsep=0pt\parsep=0pt\partopsep=0in\usecounter{itemi}}%
   }{\end{list}\addvspace{0pt}}

\makeatletter
\let\corollary\@undefined
\let\c@corollary\@undefined
\let\endcorollary\@undefined
\let\definition\@undefined
\let\c@definition\@undefined
\let\enddefinition\@undefined
\let\theorem\@undefined
\let\c@theorem\@undefined
\let\endtheorem\@undefined
\let\lemma\@undefined
\let\c@lemma\@undefined
\let\endlemma\@undefined
\makeatother

\newtheoremstyle{tightItalic}
  {0.5\myskip}
  {0.5\myskip}
  {}
  {}
  {\itshape}
  {.}
  { }
  {}

\newtheoremstyle{tightBf}
  {0.5\myskip}
  {0.5\myskip}
  {}
  {}
  {\bf}
  {.}
  {.5em}
  {}

\newtheoremstyle{tightBBf}
  {0.5\myskip}
  {0.5\myskip}
  {}
  {}
  {\bf}
  {.}
  {.5em}
  {}

\theoremstyle{tightBf}
\newtheorem{theorem}{Theorem}[section]

\newtheorem*{theorem*}{Theorem}
\newaliascnt{corollary}{theorem}%
\newtheorem{corollary}[corollary]{Corollary}
\aliascntresetthe{corollary}

\newaliascnt{definition}{theorem}%
\newtheorem{definition}[definition]{Definition}
\aliascntresetthe{definition}

\newaliascnt{statement}{theorem}%

\aliascntresetthe{statement}

\newaliascnt{lemma}{theorem}%
\newtheorem{lemma}[lemma]{Lemma}
\aliascntresetthe{lemma}

\newaliascnt{example}{theorem}%

\aliascntresetthe{example}

\newaliascnt{remark}{theorem}%
\newtheorem*{remark*}{Remark}
\aliascntresetthe{remark}

\newaliascnt{proposition}{theorem}%

\aliascntresetthe{proposition}

\newaliascnt{property}{theorem}%

\aliascntresetthe{property}

\theoremstyle{tightBBf}

\newaliascnt{problem}{theorem}%
\newtheorem{problem}[problem]{Problem}
\aliascntresetthe{problem}

\theoremstyle{tightItalic}


\usepackage[dvipsnames]{xcolor}
\usepackage{algorithm}
\usepackage{algpseudocode}
\usepackage{varwidth}
\usepackage{caption}
\usepackage[super]{nth}
\usepackage{pdfpages}

\usepackage[pagebackref=true,breaklinks=true,letterpaper=true,colorlinks,bookmarks=false]{hyperref}

\wacvfinalcopy 

\def\wacvPaperID{72} 

\begin{document}
\title{Guaranteed Parameter Estimation for Discrete Energy Minimization}

\author{Mengtian Li\\
Carnegie Mellon University\\
{\tt\small mtli@cs.cmu.edu}
\and
Daniel Huber\\
Carnegie Mellon University\\
{\tt\small dhuber@cs.cmu.edu}
}

\twocolumn[{%
\renewcommand\twocolumn[1][]{#1}%
\maketitle
\begin{center}
    \centering
    \includegraphics[trim={0 5cm 3cm 0}, clip, width=0.86\textwidth]{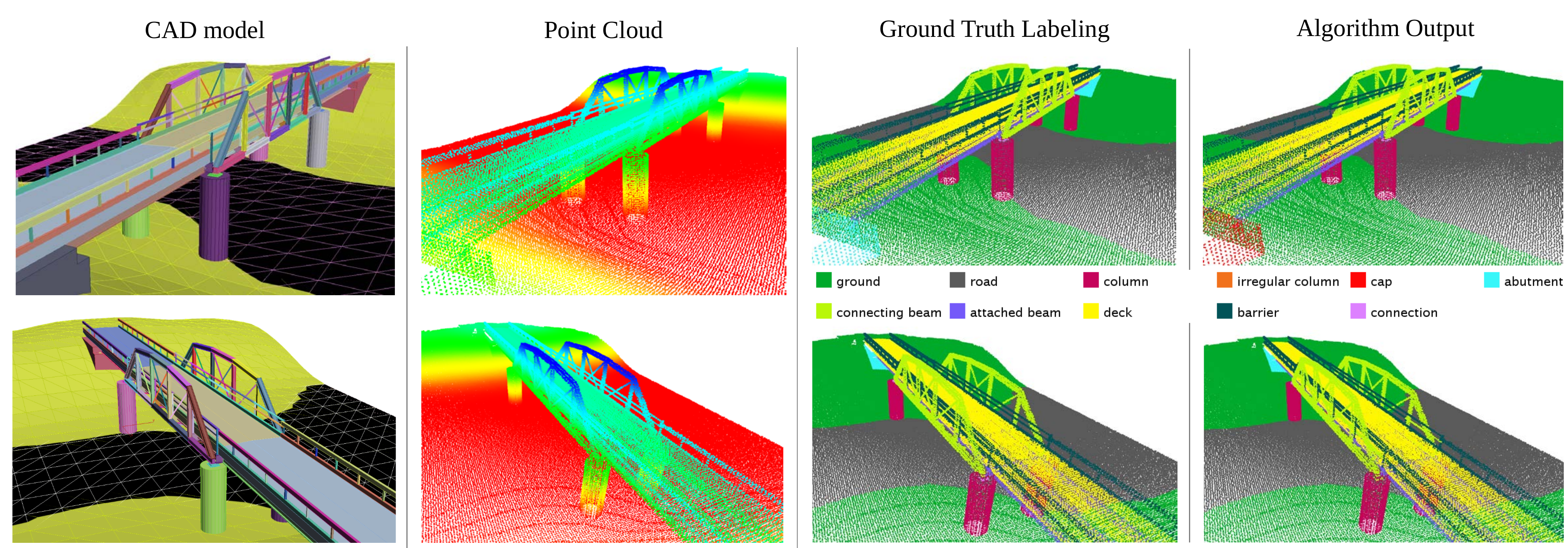}
   \captionof{figure}{Semantic labeling of a large-scale outdoor scene. We propose a generic structural learning algorithm with theoretical guarantees. When applied to scene parsing on the Cornell RGB-D dataset~\cite{koppula2011semantic, anand2012contextually}, it runs three times faster than the competing method while keeping the same level of accuracy. On a larger scale problem of bridge component recognition, our algorithm solves the scene parsing problem intractable to previous methods. The point cloud dataset we created contains 11 domain-specific semantic class and is generated by merging several simulated LiDAR scans taken from multiple locations in the CAD model scene.}
    \label{fig:semlabel}
    \bigskip
\end{center}%
}]

\begin{abstract}
Structural learning, a method to estimate the parameters for discrete energy minimization, has been proven to be effective in solving computer vision problems, especially in 3D scene parsing. As the complexity of the models increases, structural learning algorithms turn to approximate inference to retain tractability. Unfortunately, such methods often fail because the approximation can be arbitrarily poor. In this work, we propose a method to overcome this limitation through exploiting the properties of the joint problem of training time inference and learning. With the help of the learning framework, we transform the inapproximable inference problem into a polynomial time solvable one, thereby enabling tractable exact inference while still allowing an arbitrary graph structure and full potential interactions. Our learning algorithm is guaranteed to return a solution with a bounded error to the global optimal within the feasible parameter space. We demonstrate the effectiveness of this method on two point cloud scene parsing datasets. Our approach runs much faster and solves a problem that is intractable for previous, well-known approaches.
\end{abstract}

\vspace{-5.4em}
\section{Introduction}



With the increased accessibility of 3D sensing, demand is increasing for efficient methods to transform 3D data into higher level, semantically relevant representations.  Many of the most popular and successful 3D scene parsing algorithms can be reduced to some form of discrete energy minimization (or energy minimization for short) \cite{anand2012contextually, armeni20163d, hedau2009recovering, Silberman:ECCV12, schwing2012efficient, shapovalov2011cutting, silberman2011indoor, xiong2010using}.
One of the benefits of energy minimization methods is that they are able to capture contextual information or to encode prior knowledge. These capabilities are particularly important in complex 3D scene parsing, where local cues may be insufficient.
For example, in the task of bridge component recognition (Figure \ref{fig:semlabel}), attached beams have similar appearance to connecting beams. The difference is that attached beams are usually beneath the deck and on top of connections, whereas connecting beams are not. Therefore, to tell these two classes apart, the scene parsing algorithms need to incorporate knowledge of how a bridge is typically built, which governs the spatial relationships of the components.
For another example, in 3D indoor scene parsing \cite{xiong2010using}, coplanarity of two planes fitted on point clouds is a strong cue for them to be labeled as ``wall.'' In contrast, the same coplanarity might not be useful if one of them is labeled as clutter. So the existence of certain features on a pair of nodes in the graph encourages certain joint labeling of the two nodes. These relationships can depend on the feature, the label configuration, and the particular edge. In order to encode the interactions, we need a parametrized energy function with a large parameter space~\footnote{Note that the simple and popular smoothing prior model of energy minimization \cite{boykov2006graph} is unable to capture such sophisticated interactions.}. An immediate question with such formulation is how to estimate these parameters autonomously.

Parameter estimation for energy minimization, also called structural learning, fails when the input data becomes large and complex, due to the intractable inference subroutine. Such intractability arises, for example, in 3D scene parsing of complex structures, where a scene can be composed of hundreds or thousands of objects with arbitrary connectivity. For these problems, it might not be possible to solve the inference subroutine exactly or even to approximate to a certain precision. However, the inference subroutine, or the separation oracle to be precise, plays the important role of finding the subgradients of the objective in a structural learning framework. Using unbounded approximation for the separation oracle generates imperfect gradients, causing the learning algorithm to fail, since the quality can be arbitrarily poor~\cite{finley2008training}. Commonly, in structural learning, the inference subroutine is treated as a modular ``black box,'' but that approach leads to an intractable formulation.

In this paper, we show that considering together the joint problem of the overarching training and the inference subroutine enables us to exploit properties that would not be possible otherwise. Specifically, we make the following contributions. First, we propose a theoretically sound structural learning algorithm without the limitation of intractable inference. We review and exploit the properties of the joint problem of training time inference and learning. By modifying the training procedure, we can perform a training time inference corresponding to a binary submodular problem that is much easier than the original one while keeping the testing time inference problem almost the same. This method can be extended to learn higher order potentials as well. Second, while making no assumptions on the structure of the graph or on the potential type, we prove that our algorithm returns a solution within a given absolute error relative to the global optimal within the feasible parameter space. In addition, we demonstrate our algorithm's performance on two 3D scene parsing datasets. On one dataset, our algorithm runs three times faster than the competing method \cite{anand2012contextually} and achieves the same level of accuracy. Our algorithm finds a solution efficiently on the second, more complex problem, which is intractable for competing methods. Also, we show that what is learned by the model captures domain knowledge and is easily interpretable.

\section{Related Work}


Most existing literature on structural learning is based on the max-margin formulation proposed by Taskar \etal~\cite{roller2004max}. Directly minimizing the negative log-likelihood is NP-hard for many problems, and approximation must be used.  The max-margin formulation uses a convex surrogate loss, removing the need for computing the partition function. Joachims \etal~\cite{joachims2009cutting, tsochantaridis2004support} generalized this max-margin formulation to arbitrary structural outputs, a method known as structural SVM. The concept of max-margin structural learning has been successfully applied to many problems in computer vision. These works usually have limiting assumptions: tree-like or special structure output \cite{mottaghi2014role, schwing2012efficient, yang2011articulated}, small structural space \cite{hedau2009recovering, xiong2010using}, or restricted potential type \cite{anguelov2005discriminative, munoz2009contextual, taskar2004learning, szummer2008learning}. Under these assumptions, exact inference is possible. However, we don't make these assumptions, yet we can still apply exact inference during training. Other works adopt approximate inference for the separation oracle \cite{anand2012contextually}. These methods have no guarantee of the solution quality.  Notably, a common approximation scheme is convex programming relaxation \cite{jancsary2013learning}. Our early experiments show that methods based on this type of relaxation produce results with undesirably low accuracy.

The most similar work to our approach is \cite{finley2008training}, in which they point out the problem of training structural SVMs when exact inference is not possible and proposed two workarounds. The first one is to assume a constant factor approximation of the inference procedure. However, it was shown in \cite{li2016complexity} that such an assumption is not reasonable, as the problem cannot be approximated with any meaningful guarantee. The second workaround is to use the persistency property of binary MRFs, yet there is no quality guarantee of the learned parameters. In addition, we find the approach often fails in practice. Many works \cite{lacoste2013block, ramanan2014dual, shah2015multi} focus on improving the performance of structural SVM itself, but still they face the problem of an imperfect separation oracle. 

Similar to previous works, our algorithm is based on the max-margin formulation \cite{roller2004max}. We adopt non-negative constraints to restrict the parameter space \cite{szummer2008learning, taskar2004learning}, but in combination with a different loss and a different separation oracle for tractability.

The separation oracle in structural learning is frequently solved by energy minimization. Here, we highlight energy minimization algorithms used in this work and refer readers to \cite{kappes2015comparative} for a complete overview. Boykov and Kolmologrov (BK) ~\cite{boykov2004experimental} solved MAP inference for binary MRFs with a specially optimized max-flow algorithm. Rother \etal~\cite{rother2007optimizing} proposed the Quadratic Pseudo-Boolean Optimization (QPBO) algorithm for binary problems of arbitrary potentials. They first created a different auxiliary graph, in which each original node corresponds exactly to two non-terminal nodes in the new graph. Then they ran the BK algorithm on this auxiliary graph. Note that some nodes will remain unlabeled if the corresponding non-terminal node pair has conflicting assignments. For multi-class problems of arbitrary potentials, Kolmologrov~\cite{kolmogorov2006convergent} built a convergent version of the tree-reweighted max-product message passing algorithm (TRW-S). By creating a proper local polytope, an energy minimization problem can be reduced to an integer linear programming (ILP) problem \cite{werner2007linear}, and the integral constraint can be removed to derive an approximation algorithm (LP).

\section{Our Approach}

In this paper, we propose a max-margin structural learning algorithm for a pairwise model with a linear discriminant function. Our algorithm enables tractable exact training time inference through our submodular formulation, which leads to a guaranteed solution quality. Submodularity cannot be easily enforced because it requires a binary problem and limits the potential type. 
As adopted in standard machine learning algorithms, multi-class classification can be solved by training a set of 1-vs-all binary classifiers and post-processing the classifier output to make a final one-hot prediction where only a single class is labeled for each example. We adopt a similar idea. During training, we solve a set of binary classification problems but without resolving the conflicts among the binary classifiers. This setup can still learn the desired parameters, since the loss will encourage the parameters to make one-hot predictions. During testing, we enforce one-hot prediction by adding a hard constraint to the inference problem. Because we are enforcing the submoduarity on the transformed binary problems, the potential type of the original energy is not constrained.
The rest of this section introduces the desired theoretical properties of the inference procedure and the learning framework before showing our modifications to exploit these properties to build to our structural learning algorithm.

\subsection{Problems and Properties} \label{sec:pnp}

In this subsection, we first review the energy minimization formulation and the submodular property. Then we introduce our testing and training formulation.

\begin{problem}{\bf Discrete Energy Minimization} \label{prob:energy}
\begin{itemize}
\item Given a graph $\G = (\V, \E)$, define the {\em energy function}
\begin{align}
U({\bf y}) = \sum_{u\in\V} {{U_u}({y_u})} + \sum_{(u,v)\in\E} {{U_{uv}}({y_u},{y_v})},
\end{align}
where ${U_{uv}}({y_u},{y_v}) = {U_{vu}}({y_v},{y_u})$.
\item {\em Energy minimization} assigns to each node a label from a finite label set $\mathcal{L}$ to minimize the energy
\begin{align}
{\bf y}^*=\argmin_{\bf y\in\mathcal{L}^{|\V|}} U({\bf y}).
\end{align}
\end{itemize}
\end{problem}

\begin{definition}[\cite{rother2007optimizing}]
A binary (two-class) energy minimization problem is {\em submodular} if $\forall{u,v\in\V}$
\begin{align} \label{eq:submodular}
U_{uv}(0, 1) + U_{uv}(1, 0) \geq U_{uv}(0, 0) + U_{uv}(1, 1).
\end{align}
\end{definition}

It is well-known that if the energy is submodular, the global minimum can be found in polynomial time using graph cut. For multi-class problems, submodularity \cite{ramalingam2008exact} is hard to exploit due to the order dependency and magnitude constraint. The definition of submodularity requires the label set to be a totally ordered set, \eg, a depth value from 0 to 255. This definition also constrains the relative magnitude of potentials on the same edge as in the binary case. These two conditions are not generally applicable.

Another interesting property, which is exploited by  \cite{finley2008training}, is {\em persistency}, or {\em partial optimality}. Comparing to submodularity, persistency is an optimality indicator rather than an optimality guarantee. If we run the QPBO algorithm \cite{rother2007optimizing} on binary problems with arbitrary potentials, some nodes will be left unlabeled, but labelled nodes are part of the globally optimal solution. Boros \etal~\cite{boros2002pseudo} showed that in an equivalent linear programming formulation, all variables corresponding to the unlabeled nodes take 0.5 in optimal solution. Let's assume we accept relaxed ([0, 1] instead of \{0, 1\}) solutions, then running QPBO and replacing the unlabeled nodes with 0.5 will result in an approximation algorithm, which we denote as QPBO-R. 

An immediate question is how good the QPBO-R approximation is. This question is answered from a more general perspective in \cite{li2016complexity}: assuming P $\neq$
NP, for binary energy minimization in general, there does not exist a constant ratio approximation algorithm or even one with a ratio subexponential in the input size. Unfortunately, the theoretical properties of many structural learning algorithms \cite{finley2008training, lacoste2013block, shah2015multi} depend on a separation oracle with at least a constant ratio approximation, and the finding in \cite{li2016complexity} makes pointless the assumption along with the derived properties for these algorithms when applied to energy minimization in general.

We use full potential structural prediction as our testing time formulation.

\begin{problem}{\bf Full Potential Structural Prediction} \label{prob:inference}
\begin{itemize}
\item Given a node feature extractor $\phi(\cdot)$, an edge feature extractor $\phi(\cdot, \cdot)$ and a vector of weights ${\bf w}$, $\forall{k, l\in\mathcal{L}}$ define the {\em unary} and {\em pairwise potentials}
\begin{gather}
{U_u}(y_u = k) := -{\bf w}_u^k \cdot \phi(u), \\
{U_p}(y_u = k,y_v = l) := -{\bf w}_{uv}^{kl} \cdot \phi(u,v).
\end{gather}
\item Denote the graph $\G$ as ${\bf x}$, and define the {\em linear discriminant function (score function)}

\begin{align}
f({\bf x},{\bf y}) := -U({\bf y}) = {\bf w}^\intercal \Psi ({\bf x},{\bf y}).
\end{align}
\item $\Psi({\bf x},{\bf y})$ is called the {\em joint feature map}. Using {\em binary encoding} $y_u^k = \delta(y_u = k)$, $\Psi({\bf x},{\bf y})$ can be decomposed as follows:
\begin{gather}
\Psi ({\bf x},{\bf y})_{{\bf w}_u^k} = \sum_{u\in\V} y_u^k\phi(u), \\
\Psi ({\bf x},{\bf y})_{{\bf w}_{uv}^{kl}} = \sum_{(u,v)\in\E} y_u^k y_v^l \phi(u,v). \label{eq:8}
\end{gather}
\item Then the {\em testing time inference problem} is
\begin{align}
\hat{\bf y} = \argmax_{\bf y\in\mathcal{L}^{|\V|}} f({\bf x},{\bf y}) = \argmin_{\bf y\in\mathcal{L}^{|\V|}} U({\bf y}).
\end{align}
\item By abuse of notation, let $({\bf x}_i,{\bf y}_i)$ be an {\em example} from a {\em dataset} $\D = \{({\bf x}_i,{\bf y}_i)\}_{i=1}^n$.
\end{itemize}
\end{problem}

The potentials depend on both the parameters and the features, so given $\bf w$, $f = {\bf w}^\intercal \Psi ({\bf x}_i,{\bf y}_i)$ defines an energy function for an example ${\bf x}_i$. An ideal set of parameters should put the ground truth at or close to the place of lowest energy/highest score for each example so that the output of testing time inference is at or close to the ground truth. A linear score function makes the parameter estimation easier than non-linear forms. For some structural learning algorithms, kernel tricks can be applied to capture complicated mappings \cite{joachims2009cutting}.

\textbf{Full Potential Interaction} 
Notice here we have a full potential matrix ${U_p}(y_u^k,y_v^l)$ for each edge. This generalizes the well-known Potts model and associative Markov networks \cite{taskar2004learning}, where only the diagonal terms are non-zero. The relative magnitude of diagonal terms and off-diagonal terms can be arbitrary. {\em This implies that the model is more expressive as it can be both attractive (modeling a smoothing prior) or repulsive. Moreover, the potential matrix does not need to be symmetric.} Thus, such a formulation is able to encode directed relationships like relative positions, \eg, a computer monitor is usually placed above desk.

Next, we present the standard learning framework before presenting our modifications.

\begin{algorithm*}
\caption{Submodular Structural SVM for Non-submodular Problems}
\label{alg:1}
\begin{algorithmic}[1]
\State $\W \gets \varnothing$ \Comment{A working set of worst violators}
\State $\eta \gets \infty$ \Comment{The new violation in each
iteration}
\State $\xi \gets 0$ \Comment{The violation of the entire working set}
\While{$\eta - \xi > \epsilon$}
    \State
    \begin{varwidth}[t]{\linewidth}
    $({\bf w}, \xi) \gets \argmin_{{\bf w}, \xi \geq 0} \quad \frac{1}{2}||\bf w||^2 + C\xi$ \par
    \hskip\algorithmicindent
    $\qquad \quad \,$ s.t. $\quad \forall{(\bar{\bf y}_1, ...,\bar{\bf y}_n)\in\W}, \quad \frac{1}{n}{\bf w}^\intercal\sum_{i=1}^n \left[\Psi({\bf x}_i,{\bf y}_i) - \Psi({\bf x}_i,{\bar{\bf y}}_i) \right] \geq \frac{1}{n}\sum_{i=1}^n\Delta_b ({\bf y}_i,{\bar{\bf y}}_i) - \xi$ \par
    \hskip\algorithmicindent
    $\qquad \qquad \qquad \forall{j \in P}, \quad w_j \geq 0$
    \end{varwidth}
    
    \For{i = 1,...,n}
        \State ${\bar{\bf y}}_i \gets \argmax_{\hat{\bf y} \in \Y} \Delta_b ({\bf y}_i,\hat{\bf y}) + {\bf w}^\intercal \Psi ({\bf x}_i,\hat{\bf y})$ \Comment{Exact inference is now possible}
    \EndFor
    \State $\W \gets \W \cup \{(\bar{\bf y}_1, ...,\bar{\bf y}_n)\}$
    \State $\eta \gets \frac{1}{n}\sum_{i=1}^n\Delta_b ({\bf y}_i,{\bar{\bf y}}_i) -\frac{1}{n}{\bf w}^\intercal\sum_{i=1}^n \left[\Psi({\bf x}_i,{\bf y}_i) - \Psi({\bf x}_i,{\bar{\bf y}}_i) \right]$
\EndWhile
\State \textbf{return} $\bf w$
\end{algorithmic}
\end{algorithm*}

\begin{problem}{\bf Structural SVM} \cite{joachims2009cutting}  \label{prob:ssvm}
\begin{itemize}
\begin{align}
\min_{{\bf w}, \xi \geq 0} \quad
&\frac{1}{2}||\bf w||^2 + C\xi \\
\text{s.t.} \quad &\forall{(\bar{\bf y}_1, ...,\bar{\bf y}_n)\in\Y^n}: \notag
\end{align}
\vspace{-2.6em}
\begin{align}
\frac{1}{n}{\bf w}^\intercal\sum_{i=1}^n \left(\Psi - \bar{\Psi} \right) \geq \frac{1}{n}\sum_{i=1}^n\Delta ({\bf y}_i,{\bar{\bf y}}_i) - \xi, 
\label{eq:ssvmcon}
\end{align}
\end{itemize}
where $\Psi$ and $\bar{\Psi}$ are shorthand for $\Psi({\bf x}_i,{\bf y}_i)$ and $\Psi({\bf x}_i,{\bar{\bf y}}_i)$.
\end{problem}


Structural SVMs are an extension to standard SVMs for structural outputs. A structural SVM finds the optimal set of parameters that creates a large margin relative to the loss for each structural example in the dataset. Here $C$ is the parameter that controls the relative weighting between regularization and risk minimization, and $\Delta ({\bf y}_i,\hat{\bf y})$ is a loss function encoding the penalty for a wrong labeling.

Due to the combinatorial nature of the label space ($\Y = \mathcal{L}^{|\V|}$) , its size, \ie, the number of constraints (\ref{eq:ssvmcon}) is exponential. Joachims \etal~\cite{joachims2009cutting, tsochantaridis2004support} proposed the cutting-plane algorithm, which finds the optimal solution by adding only a polynomial number of constraints, given a separation oracle to compute the subgradients.

\begin{definition}
Given a loss function $\Delta ({\bf y}_i,\hat{\bf y})$, the {\em loss augmented inference} or {\em separation oracle} is a procedure that finds
\begin{align}
{\bar{\bf y}}_i = \argmax_{\hat{\bf y} \in \Y} \Delta ({\bf y}_i,\hat{\bf y}) + {\bf w}^\intercal \Psi ({\bf x}_i,\hat{\bf y}).
\end{align}
\end{definition}

The loss augmented inference finds the worst violators of the margin. Instead of bounding in the entire structural space $\hat{\bf y} \in \Y$, the cutting-plane algorithm bounds the violation of the worst violators. It can be shown that this is equivalent to solving the original problem, but now the algorithm terminates in polynomial time and returns a globally optimal solution.

\subsection{The Joint Problem for Parameter Estimation}

This subsection describes our modifications to solve the joint problem that is not limited by the intractable separation oracle as in previous approaches.
For the loss fuction, we use {\em Hamming loss} with the goal of labeling each node in the graph correctly:
\begin{align}
\Delta ({\bf y},\bar{\bf y}) =\rho \left[1-\frac{1}{|\V|}\sum_{u\in\V}\delta(y_u = {\bar y}_u) \right].
\end{align}

The loss equals to (1 - accuracy) scaled by a factor $\rho$. The structure of the loss is simple, and the loss can be merged into the unary potentials, making loss augmented problem the same problem as Problem \ref{prob:inference}. 

\subsubsection{Multi-class to Binary Transformation}

For loss augmented inference, we use a binary encoding and remove the sum-up-to-1 constraint ($\sum_{k\in\mathcal{L}} y_u^k = 1$) to use the graph-cut algorithm \cite{rother2007optimizing}. The loss also needs to be slightly modified to address the removal of the constraint. We adopt the {\em Hamming loss for binary encoding}:

\begin{align}
\Delta_b ({\bf y},\bar{\bf y}) = \frac{\rho}{2|\V|}\sum_{u\in\V}\sum_{k\in\mathcal{L}}\delta(y_u^k \neq {\bar y}_u^k)
\label{eq:lossb}.
\end{align}

The above modifications are based on the following observations:
\begin{itemize}
\item With the sum-up-to-1 constraint, $\Delta({\bf y},\bar{\bf y})$ and $\Delta_b({\bf y},\bar{\bf y})$ are equivalent;
\item Without the sum-up-to-1 constraint, let $\delta(y_u = \bar{y_u}) = \prod_{k\in\mathcal{L}} \delta(y_u^k = {\bar y}_u^k )  $, then $\Delta_b({\bf y},\bar{\bf y})$ is a tight upper bound of  $\Delta({\bf y},\bar{\bf y})$ in that $\Delta_b({\bf y},\bar{\bf y}) \geq \Delta({\bf y},\bar{\bf y})$ and $\Delta_b({\bf y},\bar{\bf y}) = 0$ if and only if  $\Delta({\bf y},\bar{\bf y}) = 0$;
\end{itemize}

In our approach, the removal of the sum-up-to-1 constraint changes the separation oracle, and the binary labeling might not have a consistent interpretation of the original labeling during training. However, the tightness 
of the loss function shows that we are effectively learning parameters to minimize the original loss.
The sum-up-to-1 constraint is implicitly enforced in a soft manner through the loss minimization during training. Soft labeling ($y_u^k \in [0, 1]$) is adopted in \cite{anand2012contextually, finley2008training}. In this case, the loss is defined by replacing $\delta(y_u^k \neq {\bar y}_u^k)$ with $|y_u^k - {\bar y}_u^k|$ in
(\ref{eq:lossb}). 
In contrast to the hard labeling that we use, for soft labeling without the sum-up-to-1 constraint, $\Delta_b ({\bf y},\bar{\bf y})$ does not have the same property of being a tight upper bound.

\subsubsection{Enforcing Submodularity}

As presented in Section \ref{sec:pnp}, without any relaxation, the transformed binary problem puts 
great challenge to the inference subroutine because the problem is NP-hard and not even possible to approximate with a guarantee. Thus, we need to enforce submodularity to enable tractable exact inference.

The transformed binary problem $U^b$ takes the form
\begin{align}
U_p^b(y_u^k, y_u^l) = y_u^k  y_u^l {U_p}(y_u = k,y_v = l).
\end{align}
Note that it does not have a full potential matrix, and only $U_p^b(1, 1)$ can be nonzero. If, for all edges, $U_p^b(1, 1)$ is non-positive, the whole energy satisfies (\ref{eq:submodular}) and is submodular. Since our algorithm depends on only $U_p^b(1, 1)$ being non-zero, the multi-class-to-binary transformation must also be applied to binary classification problems, which is not necessary in the typical 1-vs-all setup.

One way to satisfy the condition of $U_p^b(1, 1) \leq 0$ is to have all edge features $\phi(\cdot, \cdot)$ and pairwise parameters ${\bf w}_{uv}^{kl}$ be non-negative. It is reasonable to assume pairwise features can be always non-negative, since in many applications, the features are normalized to [0, 1] during a pre-processing step. Therefore, we add additional constraints only on the weights (\ref{eq:nonneg}). We summarize our formulation as follows:

\begin{problem}{\bf Partially Non-negative Structural SVM} \label{prob:nonegssvm}
\begin{itemize}
\begin{align}
\min_{{\bf w}, \xi \geq 0} \quad
&\frac{1}{2}||\bf w||^2 + C\xi \\
\text{s.t.} \quad &\forall{(\bar{\bf y}_1, ...,\bar{\bf y}_n)\in\Y^n}: \notag
\end{align}
\vspace{-2.3em}
\begin{align}
\frac{1}{n}{\bf w}^\intercal\sum_{i=1}^n \left(\Psi - \bar{\Psi} \right) \geq \frac{1}{n}\sum_{i=1}^n\Delta_b ({\bf y}_i,{\bar{\bf y}}_i) - \xi
\end{align}
\vspace{-1.9em}
\begin{align}
\qquad \forall{j \in P}, \quad w_j \geq 0
\label{eq:nonneg}
\end{align}
\end{itemize}
where $\Psi$ and $\bar{\Psi}$ are short for $\Psi({\bf x}_i,{\bf y}_i)$ and $\Psi({\bf x}_i,{\bar{\bf y}}_i)$. $P$ is the set of indices where the parameter should be non-negative, \eg, the pairwise weights. 
\end{problem}

To solve this problem, we adopt the standard max-margin formulation. Our complete algorithm is shown in Algorithm \ref{alg:1}.

\subsubsection{Solving the Modified Quadratic Program}

Non-negative constraints have been previously employed in structural learning but in a different context. In pose estimation \cite{yang2011articulated, zhu2012face}, the quadratic spring terms must be non-negative. These works employ a tree-structured model, so exact inference is possible through dynamic programming. It is shown in \cite{ramanan2014dual} that for solvers in the primal space, adding non-negative constraints amounts to clipping the parameters during the update step while leaving the rest unchanged. We adopt the dual coordinate descent solver from \cite{ramanan2014dual} to solve the minimization problem in Problem \ref{prob:nonegssvm}. In practice, however, we find that a commercial general purpose QP solver, namely Gurobi \cite{gurobi2015}, is several times faster under the same tolerance setting.

\subsection{Generalization to Higher Order Potentials}

Higher order potentials capture more interactions than the pairwise potentials. For example, a column between a pair of abutments is a \nth{3} order interaction. Our generalization is based on the pairwise reduction from arbitrary high order potentials proposed by Ishikawa \etal~\cite{ishikawa2011transformation}. Taking the \nth{3} order case as an example, the reduction is based on the identity over Boolean variables
\begin{align}
    -xyz = \min_{w \in \{0, 1\}} -w(x+y+z-2).
\end{align}
If the \nth{3} order potential is non-positive, then the constructed pairwise potentials in the reduction are also non-positive and vice versa. This enables us to enforce submodularity on \nth{3} order energy minimization problems. Likewise, we can apply similar constraints for even higher order problems. Details for general higher order can be found in the supplementary material.

\section{Analysis}

The following theorems prove that our algorithm is both efficient and globally optimal.

\begin{theorem} {\bf Correctness of the algorithm} For any training datasets $\D$ and any $\epsilon > 0$, if $({\bf w}^*, \xi^*)$ is the optimal solution of Problem \ref{prob:nonegssvm}, then Algorithm \ref{alg:1} returns a solution $({\bf w}, \xi)$ that has a better objective value than $({\bf w}^*, \xi^*)$, and for which $({\bf w}, \xi+\epsilon)$ is feasible in Problem \ref{prob:nonegssvm}.
\end{theorem}
\begin{proof}
The original proof presented in \cite{joachims2009cutting} holds, since it does not depend on any constraints involving only $\bf w$, and in our case, all separation oracles during training are exact.
\end{proof}

\begin{theorem} {\bf Convergence of the algorithm} Algorithm \ref{alg:1} terminates in polynomial time.
\end{theorem}

The proof is provided in the supplementary material. Briefly, the separation oracle terminates in polynomial time, and adding negative constraints does not change the nature of the convex optimization in line 5. Note that the actual convergence rate depends on the QP solver used for line 5.


\section{Testing Time Inference}

While we have a transformed and restricted problem during training, during testing we might still have a full potential matrix for each potential. The only limitation in the expressiveness of the formulation is that all the pairwise potentials are non-positive (in the sense of minimization). We show in our experiments that this restriction has limited effect on the overall accuracy. At testing time, the inference is performed independently on each example, and the error does not accumulate as it does at training time. If the graph is small or sparse, exact inference is possible through ILP. Otherwise, TRW-S \cite{kolmogorov2006convergent} provides good approximation in practice  for general potentials \cite{kappes2015comparative}.

\begin{figure}[b]
\begin{center}
   \includegraphics[trim={2cm 6cm 3cm 7cm}, clip, width=1\linewidth]{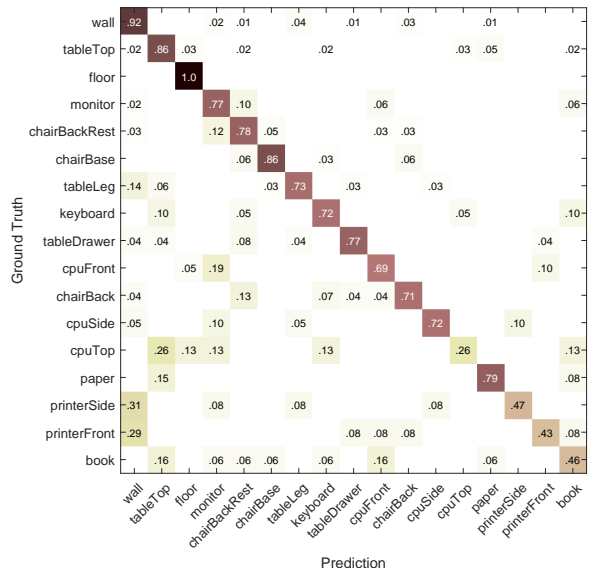}
\end{center}
   \caption{Confusion matrix of our algorithm on the Cornell RGB-D Dataset (office scenes).}
\label{fig:office}
\end{figure}

\begin{table}[t]
\centering
\begin{tabular}{l|lllll} \hline
     & Accu & macro P & macro R & Time & Speedup\\ \hline
\cite{anand2012contextually} & 81.45     & 76.79   & 70.07   & 4.11h & 1.00\\
Ours                    & 80.72     & 73.42   & 69.74   & 1.34h & \textbf{3.06}\\ \hline
 
\end{tabular}
\vspace{2mm}
\caption{Performance comparison on the Cornell RGB-D Dataset (office scenes). The second column denotes the overall accuracy. The 'P' and 'R' here stand for precision and recall respectively. As defined in \cite{anand2012contextually}, the macro P or R equates to class average P or R.}
\label{tab:cornell}
\end{table}

\section{Experiments}

We demonstrate the performance of our algorithm on the standard Cornell RGB-D dataset and a larger scale bridge dataset, which we created. On Cornell's dataset, our algorithm runs three times faster while keeping the same level of accuracy as the competing method. On the bridge dataset, the competing methods are unable to solve the scene parsing problem due to the intractable seperation oracle. In contrast, our algorithm is able to solve it efficiently and accurately. In addition, we visualize the weights learned by our algorithm to show that our model captures domain knowledge.

\subsection{Cornell RGB-D Dataset: Understanding 3D Scenes}

The Cornell RGB-D dataset \cite{koppula2011semantic, anand2012contextually} is an indoor point cloud dataset captured by Microsoft Kinect. The point clouds are obtained through merging multiple RGB-D views using the simultaneous localization and mapping (SLAM) algorithm. The point clouds are clustered into multiple segments. This dataset is suitable for testing structural learning prediction algorithms because it is necessary to take into account the neighborhood interaction for each node in order to label the segments correctly.

We compare our approach with \cite{anand2012contextually} (also \cite{finley2008training}) and use the same segmentation and features to ensure a fair comparison. The pairwise features cover visual appearance, local shape and geometry, and geometric context. Their algorithm adopts the persistency based approach in \cite{finley2008training} (QPBO-R in Section \ref{sec:pnp}). Note this method has no guarantee of optimality and an empirical heuristic needs to be adopted as discussed below. A variant of their algorithm makes use of additional class label information to limit the pairwise interactions to a predefined set of classes. The method assumes some labels are parts of an object, and restricts some potentials to be only among these labels. This information is usually not available on other structural datasets, so we do not include it in our comparison. The 4-fold cross-validation results are summarized in Table \ref{tab:cornell}. The first row is taken from their paper. Our confusion matrix is shown in Figure \ref{fig:office}. Notice that even with the additional constraints, our algorithm achieves approximately the same accuracy as  \cite{anand2012contextually} in 1/3 the time and with the critical advantage of a theoretical guarantee bounding the error.


\begin{figure}[t]
\begin{center}
   \includegraphics[trim={2cm 6.5cm 3cm 7cm}, clip, width=0.95\linewidth]{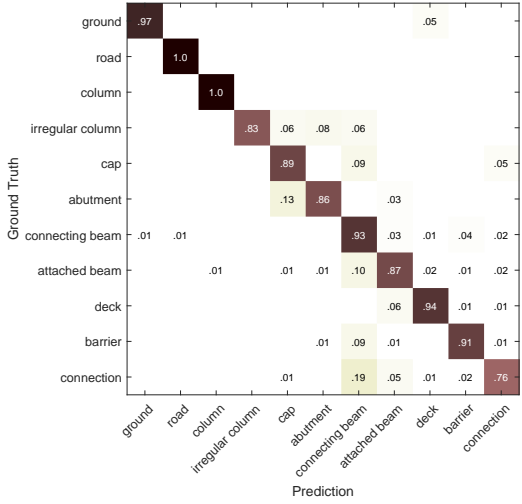}
\end{center}
   \caption{Confusion matrix of our algorithm on the bridge dataset.}
\label{fig:bridgeres}
\end{figure}


The competing method's implementation uses an undocumented heuristic that is vital for the learning procedure. In our algorithm, there is no need for this heuristic, because no relaxation is involved. Recall the rationale for interpreting an unlabeled node as 0.5 in Section 3. To compute the joint feature map $\Psi ({\bf x},{\bf y})$, we need to compute $y_u^k y_v^l$ in (\ref{eq:8}). If both are unlabeled, then $y_u^k y_v^l$ would be 0.25. In \cite{anand2012contextually}, an additional measure is taken when neither side is labeled by QPBO:
\begin{itemize}
\item $y_u^k y_v^l$ is interpreted as 0.5, if the coefficient, \ie, $U_p^b(y_u^k, y_u^l)$, is positive;
\item $y_u^k y_v^l$ is interpreted as 0, otherwise.
\end{itemize}

We found that without this rounding heuristic, the learning algorithm in \cite{anand2012contextually} terminates after a dozen or fewer iterations with a newly found violation smaller than the violation of the current working set, which is impossible if the loss augmented inference is exact. Such early termination prevents the structral SVM from learning any meaningful potentials, and the prediction is usually a failure. This effect has been observed using both their implementation and our independent implementation on Cornell's RGB-D Dataset and the bridge dataset in next subsection.

\subsection{Bridge Dataset: Scaling up to Complex Structures} \label{sec:bridgedataset}

\begin{figure}[t]
\begin{center}
   \includegraphics[trim={2cm 7.5cm 3cm 7cm}, clip, width=0.95\linewidth]{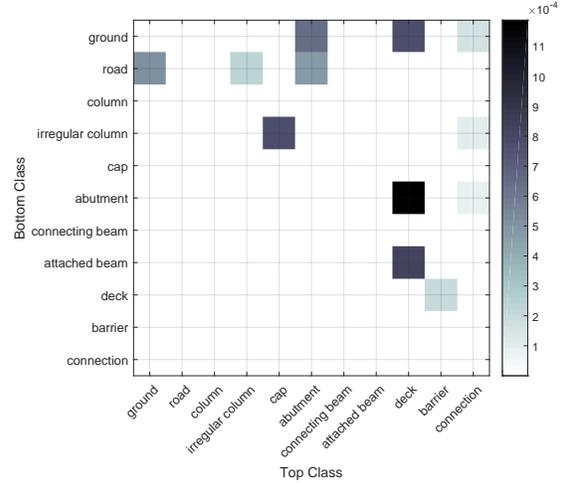}
\end{center}
   \caption{The pairwise weights for the {\em on-top-of} feature. These weights capture domain knowledge for bridge architecture.}
\label{fig:wsupport}
\end{figure}

\begin{figure*}[t]
\begin{center}
   \includegraphics[trim={0 12cm 0 0}, clip, width=0.70\linewidth]{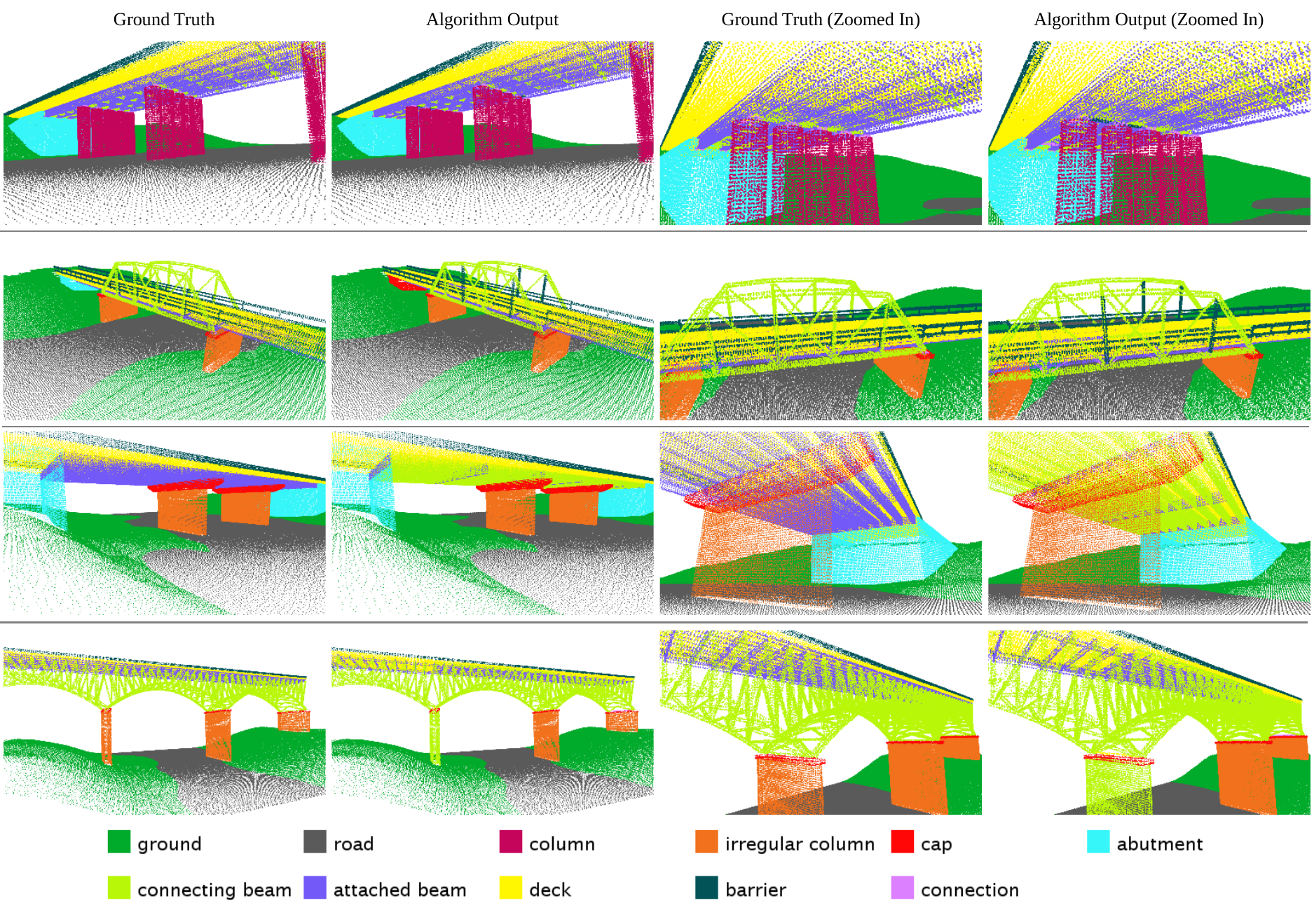}
\end{center}
   \caption{Output of our algorithm on the bridge dataset. Some errors can be seen by comparing the \nth{3} and \nth{4} columns.}
\label{fig:qualitative}
\end{figure*}

For a second experiment, we tested out our algorithm on a domain-specific dataset to evaluate its performance against a large dataset with complex structures. To this aim, we created a synthetic but realistic bridge dataset (Figure \ref{fig:semlabel}) modeling complicated building structure. Such a dataset is useful for developing 3D reverse engineering techniques, which can find their application in as-built Building Information Model (BIM) creation \cite{xiong2013automatic} and infrastructure inspection \cite{song2015}. Unlike color or RGB-D images, full building laser scan datasets are scarce, thus we utilize a realistic synthetic dataset. We constructed CAD models of bridges, and generated the point clouds by placing a virtual laser scanner, complete with a noise model, in the scene as if we are actually conducting actual field scans. Multiple scans are taken per scene and merged into a single point cloud. In total, we have 25 bridge models of five different types. Each model contains 200K to 500K 3D points after down-sampling.

Similar to the Cornell RGB-D dataset, the task is to semantically label the segments, and we define eleven semantic classes for this dataset. We train a random forest classifier on SHOT descriptors \cite{tombari2010unique} to obtain a label class distribution for each point. The descriptor encodes histogram of local surface information. We take the mean class distribution as the node feature for each segment. We use ground truth segmentation for benchmarking the contextual classification algorithms. We build a graph based on the physical adjacency of the segments and use on-top-of, principal direction consistency, and perpendicularity as three edge features. The accuracy is computed at the node level. On average, the bridge scenes contain ten times more segments and nine times more edges than the Cornell RGB-D dataset. We split the dataset into five folds, each containing five bridge models. 

The cross-validation result is summarized in Figure \ref{fig:bridgeres} and visualized in Figure \ref{fig:qualitative}. We obtain 90.07\% overall accuracy for semantic labeling the scene with 11 classes. For a single fold, the training takes 1.3 hours, and testing takes 89 seconds for five scenes. We attempted to use 
\cite{lacoste2013block} and \cite{anand2012contextually} as competing methods. However, the first fails due to the poor separation oracle and the latter could not handle this large scale of data and did not terminate after 7.5 days.



\noindent \textbf{Capturing domain knowledge.} Our algorithm is able to encode domain knowledge in the pairwise weights. For instance, we visualize the weights for the {\em on-top-of} feature in Figure \ref{fig:wsupport}. The feature is a binary indicator, and the product of this feature and the corresponding weight adds towards the overall score. The matrix reveals typical structural relationships seen in bridge architecture, \eg, the abutment and attached beam are usually placed beneath the deck.





\section{Conclusion}

In this work, we propose a method to overcome the problem caused by using unbounded approximation for the separation oracle in structural learning. We show theoretically that after properly exploiting the properties of the joint problem of optimizing structural SVM and the separation oracle, we can retrieve the theoretical guarantees of structural SVMs that are lost when unbounded approximation is used. The performance on the Cornell RGB-D dataset and our bridge dataset demonstrates the effectiveness and efficiency of this method.

\section*{Acknowledgements}
This material is based upon work supported by the National Science Foundation under Grant No. IIS-1328930.

{\small
\bibliographystyle{ieee}
\bibliography{egbib}

\begin{thebibliography}{10}\itemsep=-1pt

\bibitem{anand2012contextually}
A.~Anand, H.~S. Koppula, T.~Joachims, and A.~Saxena.
\newblock Contextually guided semantic labeling and search for
  three-dimensional point clouds.
\newblock {\em The International Journal of Robotics Research}, page
  0278364912461538, 2012.

\bibitem{anguelov2005discriminative}
D.~Anguelov, B.~Taskarf, V.~Chatalbashev, D.~Koller, D.~Gupta, G.~Heitz, and
  A.~Ng.
\newblock Discriminative learning of {M}arkov random fields for segmentation of
  3{D} scan data.
\newblock In {\em Computer Vision and Pattern Recognition, 2005. CVPR 2005.
  IEEE Computer Society Conference on}, volume~2, pages 169--176. IEEE, 2005.

\bibitem{armeni20163d}
I.~Armeni, O.~Sener, A.~R. Zamir, H.~Jiang, I.~Brilakis, M.~Fischer, and
  S.~Savarese.
\newblock 3{D} semantic parsing of large-scale indoor spaces.
\newblock CVPR, 2016.

\bibitem{boros2002pseudo}
E.~Boros and P.~L. Hammer.
\newblock Pseudo-boolean optimization.
\newblock {\em Discrete applied mathematics}, 123(1):155--225, 2002.

\bibitem{boykov2006graph}
Y.~Boykov and G.~Funka-Lea.
\newblock Graph cuts and efficient {ND} image segmentation.
\newblock {\em International journal of computer vision}, 70(2):109--131, 2006.

\bibitem{boykov2004experimental}
Y.~Boykov and V.~Kolmogorov.
\newblock An experimental comparison of min-cut/max-flow algorithms for energy
  minimization in vision.
\newblock {\em Pattern Analysis and Machine Intelligence, IEEE Transactions
  on}, 26(9):1124--1137, 2004.

\bibitem{finley2008training}
T.~Finley and T.~Joachims.
\newblock Training structural {SVM}s when exact inference is intractable.
\newblock In {\em Proceedings of the 25th international conference on Machine
  learning}, pages 304--311. ACM, 2008.

\bibitem{gurobi2015}
I.~Gurobi~Optimization.
\newblock Gurobi optimizer reference manual, 2015.

\bibitem{hedau2009recovering}
V.~Hedau, D.~Hoiem, and D.~Forsyth.
\newblock Recovering the spatial layout of cluttered rooms.
\newblock In {\em Computer vision, 2009 IEEE 12th international conference on},
  pages 1849--1856. IEEE, 2009.

\bibitem{ishikawa2011transformation}
H.~Ishikawa.
\newblock Transformation of general binary {MRF} minimization to the
  first-order case.
\newblock {\em Pattern Analysis and Machine Intelligence, IEEE Transactions
  on}, 33(6):1234--1249, 2011.

\bibitem{jancsary2013learning}
J.~Jancsary, S.~Nowozin, and C.~Rother.
\newblock Learning convex {QP} relaxations for structured prediction.
\newblock In {\em Proceedings of The 30th International Conference on Machine
  Learning}, pages 915--923, 2013.

\bibitem{joachims2009cutting}
T.~Joachims, T.~Finley, and C.-N.~J. Yu.
\newblock Cutting-plane training of structural {SVM}s.
\newblock {\em Machine Learning}, 77(1):27--59, 2009.

\bibitem{kappes2015comparative}
J.~H. Kappes, B.~Andres, F.~A. Hamprecht, C.~Schn{\"o}rr, S.~Nowozin, D.~Batra,
  S.~Kim, B.~X. Kausler, T.~Kr{\"o}ger, J.~Lellmann, et~al.
\newblock A comparative study of modern inference techniques for structured
  discrete energy minimization problems.
\newblock {\em International Journal of Computer Vision}, 115:155--184, 2015.

\bibitem{kolmogorov2006convergent}
V.~Kolmogorov.
\newblock Convergent tree-reweighted message passing for energy minimization.
\newblock {\em Pattern Analysis and Machine Intelligence, IEEE Transactions
  on}, 28(10):1568--1583, 2006.

\bibitem{koppula2011semantic}
H.~S. Koppula, A.~Anand, T.~Joachims, and A.~Saxena.
\newblock Semantic labeling of 3{D} point clouds for indoor scenes.
\newblock In {\em Advances in Neural Information Processing Systems}, pages
  244--252, 2011.

\bibitem{lacoste2013block}
S.~Lacoste-Julien, M.~Jaggi, M.~Schmidt, and P.~Pletscher.
\newblock Block-coordinate {F}rank-{W}olfe optimization for structural {SVM}s.
\newblock {\em Machine Learning}, 2013.

\bibitem{li2016complexity}
M.~Li, A.~Shekhovtsov, and D.~Huber.
\newblock Complexity of discrete energy minimization problems.
\newblock In {\em ECCV}, 2016.

\bibitem{mottaghi2014role}
R.~Mottaghi, X.~Chen, X.~Liu, N.-G. Cho, S.-W. Lee, S.~Fidler, R.~Urtasun,
  et~al.
\newblock The role of context for object detection and semantic segmentation in
  the wild.
\newblock In {\em Computer Vision and Pattern Recognition (CVPR), 2014 IEEE
  Conference on}, pages 891--898. IEEE, 2014.

\bibitem{munoz2009contextual}
D.~Munoz, J.~A. Bagnell, N.~Vandapel, and M.~Hebert.
\newblock Contextual classification with functional max-margin {M}arkov
  networks.
\newblock In {\em Computer Vision and Pattern Recognition, 2009. CVPR 2009.
  IEEE Conference on}, pages 975--982. IEEE, 2009.

\bibitem{Silberman:ECCV12}
P.~K. Nathan~Silberman, Derek~Hoiem and R.~Fergus.
\newblock Indoor segmentation and support inference from {RGBD} images.
\newblock In {\em ECCV}, 2012.

\bibitem{ramalingam2008exact}
S.~Ramalingam, P.~Kohli, K.~Alahari, and P.~H. Torr.
\newblock Exact inference in multi-label {CRF}s with higher order cliques.
\newblock In {\em Computer Vision and Pattern Recognition, 2008. CVPR 2008.
  IEEE Conference on}, pages 1--8. IEEE, 2008.

\bibitem{ramanan2014dual}
D.~Ramanan.
\newblock Dual coordinate solvers for large-scale structural {SVM}s.
\newblock {\em arXiv preprint arXiv:1312.1743}, 2014.

\bibitem{roller2004max}
B.~T. C. G.~D. Roller.
\newblock Max-margin {M}arkov networks.
\newblock {\em Advances in neural information processing systems}, 16:25, 2004.

\bibitem{rother2007optimizing}
C.~Rother, V.~Kolmogorov, V.~Lempitsky, and M.~Szummer.
\newblock Optimizing binary {MRF}s via extended roof duality.
\newblock In {\em Computer Vision and Pattern Recognition, 2007. CVPR'07. IEEE
  Conference on}, pages 1--8. IEEE, 2007.

\bibitem{schwing2012efficient}
A.~G. Schwing and R.~Urtasun.
\newblock Efficient exact inference for 3{D} indoor scene understanding.
\newblock In {\em Computer Vision--ECCV 2012}, pages 299--313. Springer, 2012.

\bibitem{shah2015multi}
N.~Shah, V.~Kolmogorov, and C.~H. Lampert.
\newblock A multi-plane block-coordinate {F}rank-{W}olfe algorithm for training
  structural {SVM}s with a costly max-oracle.
\newblock In {\em Computer Vision and Pattern Recognition, 2015. CVPR 2015.
  IEEE Computer Society Conference on}. IEEE, 2015.

\bibitem{shapovalov2011cutting}
R.~Shapovalov and A.~Velizhev.
\newblock Cutting-plane training of non-associative {M}arkov network for 3{D}
  point cloud segmentation.
\newblock In {\em 2011 International Conference on 3D Imaging, Modeling,
  Processing, Visualization and Transmission}, pages 1--8. IEEE, 2011.

\bibitem{silberman2011indoor}
N.~Silberman and R.~Fergus.
\newblock Indoor scene segmentation using a structured light sensor.
\newblock In {\em Computer Vision Workshops (ICCV Workshops), 2011 IEEE
  International Conference on}, pages 601--608. IEEE, 2011.

\bibitem{song2015}
M.~Song and D.~Huber.
\newblock Automatic recovery of networks of thin structures.
\newblock {\em International Conference on 3D Vision}, 2015.

\bibitem{szummer2008learning}
M.~Szummer, P.~Kohli, and D.~Hoiem.
\newblock Learning {CRF}s using graph cuts.
\newblock In {\em Computer Vision--ECCV 2008}, pages 582--595. Springer, 2008.

\bibitem{taskar2004learning}
B.~Taskar, V.~Chatalbashev, and D.~Koller.
\newblock Learning associative {M}arkov networks.
\newblock In {\em Proceedings of the twenty-first international conference on
  Machine learning}, page 102. ACM, 2004.

\bibitem{tombari2010unique}
F.~Tombari, S.~Salti, and L.~Di~Stefano.
\newblock Unique signatures of histograms for local surface description.
\newblock In {\em Computer Vision--ECCV 2010}, pages 356--369. Springer, 2010.

\bibitem{tsochantaridis2004support}
I.~Tsochantaridis, T.~Hofmann, T.~Joachims, and Y.~Altun.
\newblock Support vector machine learning for interdependent and structured
  output spaces.
\newblock In {\em Proceedings of the twenty-first international conference on
  Machine learning}, page 104. ACM, 2004.

\bibitem{werner2007linear}
T.~Werner.
\newblock A linear programming approach to max-sum problem: A review.
\newblock {\em Pattern Analysis and Machine Intelligence, IEEE Transactions
  on}, 29(7):1165--1179, 2007.

\bibitem{xiong2013automatic}
X.~Xiong, A.~Adan, B.~Akinci, and D.~Huber.
\newblock Automatic creation of semantically rich 3d building models from laser
  scanner data.
\newblock {\em Automation in Construction}, 31:325--337, 2013.

\bibitem{xiong2010using}
X.~Xiong and D.~Huber.
\newblock Using context to create semantic 3{D} models of indoor environments.
\newblock In {\em BMVC}, pages 1--11, 2010.

\bibitem{yang2011articulated}
Y.~Yang and D.~Ramanan.
\newblock Articulated pose estimation with flexible mixtures-of-parts.
\newblock In {\em Computer Vision and Pattern Recognition (CVPR), 2011 IEEE
  Conference on}, pages 1385--1392. IEEE, 2011.

\bibitem{zhu2012face}
X.~Zhu and D.~Ramanan.
\newblock Face detection, pose estimation, and landmark localization in the
  wild.
\newblock In {\em Computer Vision and Pattern Recognition (CVPR), 2012 IEEE
  Conference on}, pages 2879--2886. IEEE, 2012.

\end{thebibliography}
}


\clearpage



\section*{Supplementary material}

## A. Outline

This document contains proofs relevant to our paper. Note that the contents here are not necessary to understand the main paper. First, we show the polynomial time termination of Algorithm 1 by constructing a line search and alternating between the two dual variables. Then, we give a formal statement on the extension to higher order potentials for our algorithm.

## B. Proof for Convergence of Algorithm 1

Convergence has been proven in [4, 5] for 1-slack structural SVMs. Here, we show that similar results hold for problems with non-negative constraints. The proof constructs a line search to bound the increase in the objective in each iteration. The non-negative constraints can bring additional increase for the objective when they are activated, resulting in possibly fewer iterations. Symbols used in the proof are summarized in Table 1.

**Problem B.1. Primal QP**

Using the new notations, the QP in Algorithm 1, line 5 can be written as

$$\min_{\mathbf{w}, \xi \geq 0} \quad \frac{1}{2}||\mathbf{w}||^2 + C\xi \tag{1}$$

$$\text{s.t.} \quad \mathbf{H}^\mathsf{T}\mathbf{w} \geq l - \xi\mathbf{1}, \tag{2}$$

$$\mathbf{w}_P \geq \mathbf{0} \tag{3}$$

The Lagrangian is

$$L(\mathbf{w}, \xi, \alpha, \beta, \gamma) = \frac{1}{2}||\mathbf{w}||^2 + C\xi \tag{4}$$
$$- \alpha^\mathsf{T}[\mathbf{H}^\mathsf{T}\mathbf{w} - l + \xi\mathbf{1}] - \beta^\mathsf{T}\mathbf{w} - \gamma\xi$$

Setting the differential of $L$ with respect to $\mathbf{w}$ to zero yields

$$\mathbf{w} = \mathbf{H}\alpha + \beta \tag{5}$$

Setting the differential of $L$ with respect to $\xi$ to zero yields

$$C - \alpha^\mathsf{T}\mathbf{1} = \gamma \geq 0 \tag{6}$$

Note that we define $\beta$ to be a vector of the same length as $\mathbf{w}$ for simplicity. $(\beta)_j$ is fixed to zero for every coordinate $j$ not required to be non-negative ($j \notin P$).

| Symbols | Definitions |
|---------|-------------|
| $t$ | iteration count for Algorithm 1 |
| $h_t$ | $\frac{1}{n}\sum_{i=1}^n [\Psi(\mathbf{x}_i, \mathbf{y}_i) - \Psi(\mathbf{x}_i, \bar{\mathbf{y}}_i)]$ for all $\bar{\mathbf{y}}_i$ added in the $t$-th iteration |
| $d_t$ | $\frac{1}{n}\sum_{i=1}^n \Delta_b(\mathbf{y}_i, \bar{\mathbf{y}}_i)$ for all $\bar{\mathbf{y}}_i$ added in the $t$-th iteration |
| $\mathbf{H}$ or $\mathbf{H}_t$ | $[h_1\ h_2\ ...\ h_t]$ |
| $l$ or $l_t$ | $[d_1\ d_2\ ...\ d_t]^\mathsf{T}$ |
| $R$ | $\max_{\forall i, \bar{\mathbf{y}}} ||\Psi(\mathbf{x}_i, \mathbf{y}_i) - \Psi(\mathbf{x}_i, \bar{\mathbf{y}}_i)||_2$ |
| $\Delta$ | $\max_{\forall i, \bar{\mathbf{y}}} \Delta_b(\mathbf{y}_i, \bar{\mathbf{y}})$ |
| $\alpha$ | the dual variables for margin violation |
| $\beta$ | the dual variables for non-negativity |
| $(\mathbf{w}^*, \xi^*)$ | the optimal solution of Problem 4.1 |
| $(\alpha^*, \beta^*)$ | corresponding dual variables for $(\mathbf{w}^*, \xi^*)$ |
| $J_t(\mathbf{w})$ | the primal objective value of the QP in Algorithm 1, line 5 at the $t$-th iteration |
| $D_t(\alpha, \beta)$ | the dual objective value of the QP in Algorithm 1, line 5 at the $t$-th iteration |
| $\delta_t$ | $D_t(\alpha^*, \beta^*) - D_t(\alpha_t, \beta_t)$ |

Table 1. List of symbols for the convergence proof. (Section B)

**Problem B.2. Dual QP**

The dual problem is obtained by substituting equations (5) and (6) (KKT-conditions) into the Lagrangian

$$\max_{\alpha \geq 0, \beta \geq 0} \quad -\frac{1}{2}\alpha^\mathsf{T}\mathbf{H}^\mathsf{T}\mathbf{H}\alpha - \beta^\mathsf{T}\mathbf{H}\alpha + l^\mathsf{T}\alpha - \frac{1}{2}\beta^\mathsf{T}\beta \tag{7}$$

$$\text{s.t.} \quad \alpha^\mathsf{T}\mathbf{1} \leq C \tag{8}$$

Initially, the working set $\mathcal{W}$ is empty and $J_1 = D_1 = 0$. The trivial solution $\mathbf{w} = \mathbf{0}$ generates an upper bound $C\Delta$ for the optimality gap $\delta_t$. Next, we show that this gap can be closed through a constant increase in the dual objective in each iteration. The QP is solved by a QP solver in Algorithm 1. However, we cannot bound the change of the objective value. Instead, we resort to a series of line searches. There are two sets of dual variables, $\alpha$ and $\beta$. In each iteration, we optimize $\alpha$, keeping $\beta$ fixed, and then optimize $\beta$, keeping $\alpha$ fixed. The following lemma is introduced to bound the minimal increase in the objective with a line search in $\alpha$.

**Lemma B.3.** For any unconstrained quadratic program,

$$f(\mathbf{x}) = -\frac{1}{2}\mathbf{x}^\mathsf{T}\mathbf{A}\mathbf{x} + \mathbf{b}^\mathsf{T}\mathbf{x} \qquad (9)$$

with positive semi-definite $\mathbf{A}$, a line search starting at $\mathbf{x}$ with maximum step-size $s$ towards a direction $\mathbf{g}$, such that $\nabla f(\mathbf{x})^\mathsf{T}\mathbf{g} \geq 0$ and $\mathbf{g}^\mathsf{T}\mathbf{A}\mathbf{g} \neq 0$, increases the objective by at least

$$\max_{0 \leq \lambda \leq s}\left[f(\mathbf{x}+\lambda\mathbf{g}) - f(\mathbf{x})\right]$$
$$\geq \frac{1}{2}\min\left\{s\nabla f(\mathbf{x})^\mathsf{T}\mathbf{g}, \frac{[\nabla f(\mathbf{x})^\mathsf{T}\mathbf{g}]^2}{\mathbf{g}^\mathsf{T}\mathbf{A}\mathbf{g}}\right\} \quad (10)$$

The first case applies when $\frac{\nabla f(\mathbf{x})^\mathsf{T}\mathbf{g}}{\mathbf{g}^\mathsf{T}\mathbf{A}\mathbf{g}} > s$, while the latter applies when $\frac{\nabla f(\mathbf{x})^\mathsf{T}\mathbf{g}}{\mathbf{g}^\mathsf{T}\mathbf{A}\mathbf{g}} \leq s$.

*Proof.*

$$f(\mathbf{x}+\lambda\mathbf{g}) - f(\mathbf{x}) = -\frac{1}{2}\mathbf{g}^\mathsf{T}\mathbf{A}\mathbf{g}\lambda^2 + \nabla f(\mathbf{x})^\mathsf{T}\mathbf{g}\lambda \quad (11)$$

is a simple quadratic function in $\lambda$ restricted to $[0, s]$. When $\frac{\nabla f(\mathbf{x})^\mathsf{T}\mathbf{g}}{\mathbf{g}^\mathsf{T}\mathbf{A}\mathbf{g}} \leq s$, its optimal value is obtained at $\lambda^* = \frac{\nabla f(\mathbf{x})^\mathsf{T}\mathbf{g}}{\mathbf{g}^\mathsf{T}\mathbf{A}\mathbf{g}}$, with value $\frac{[\nabla f(\mathbf{x})^\mathsf{T}\mathbf{g}]^2}{2\mathbf{g}^\mathsf{T}\mathbf{A}\mathbf{g}}$; and when $\frac{\nabla f(\mathbf{x})^\mathsf{T}\mathbf{g}}{\mathbf{g}^\mathsf{T}\mathbf{A}\mathbf{g}} > s$, its optimal value is obtained at $\lambda^* = s$, with value $\nabla f(\mathbf{x})^\mathsf{T}\mathbf{g}s - \frac{1}{2}\mathbf{g}^\mathsf{T}\mathbf{A}\mathbf{g}s^2 \geq \frac{1}{2}s\nabla f(\mathbf{x})^\mathsf{T}\mathbf{g}$. $\qquad\square$

Consider at the beginning of iteration $(t + 1)$, $t$ constraints have been added for the QP. We want to optimize this new QP based on the previous iteration's solution $(\alpha, \beta)$. Keeping $\beta$ fixed, the line search in $\alpha$ is constructed as:

$$\tilde{\alpha}(\lambda) := [-\lambda\alpha^\mathsf{T}, \; \lambda C]^\mathsf{T}, \quad \lambda \in [0, 1] \qquad (12)$$

Note the direction $(\tilde{\alpha} = [-\alpha_t^\mathsf{T}, \; C])$ is chosen so that by construction, $\alpha + \tilde{\alpha}(\lambda)$ is always in the feasible region. In order to apply Lemma B.3, we need to bound $\nabla D^\mathsf{T}\tilde{\alpha}$ and $\tilde{\alpha}^\mathsf{T}\mathbf{H}^\mathsf{T}\mathbf{H}\tilde{\alpha}$.

Due to strong duality,

$$\frac{\partial D(\alpha, \beta)}{\partial \alpha} = l - \mathbf{H}^\mathsf{T}(\mathbf{H}\alpha + \beta) = l - \mathbf{H}^\mathsf{T}\mathbf{w}, \qquad (13)$$

and due to complementary slackness, for each non-zero component $i$ of $\alpha$,

$$\frac{\partial D(\alpha, \beta))}{\partial (\alpha)_i} = d_i - h_i^\mathsf{T}\mathbf{w} = \xi \qquad (14)$$

For $(\alpha)_t$ corresponding to the newly added constraint and some $\mu$, by construction of Algorithm 1

$$\frac{\partial D(\alpha, \beta))}{\partial \alpha_t} = d_t - h_t^\mathsf{T}\mathbf{w} = \xi + \mu \geq \xi + \varepsilon \qquad (15)$$

Therefore

$$\nabla D^\mathsf{T}\tilde{\alpha} = -\mathbf{1}^\mathsf{T}\alpha\xi + C(\xi + \mu) = C\mu \qquad (16)$$

On the other hand

$$\tilde{\alpha}^\mathsf{T}\mathbf{H}^\mathsf{T}\mathbf{H}\tilde{\alpha} = \tilde{\alpha}^\mathsf{T}\mathbf{H}_t^\mathsf{T}\mathbf{H}_t\tilde{\alpha}$$
$$= \alpha^\mathsf{T}\mathbf{H}_{t-1}^\mathsf{T}\mathbf{H}_{t-1}\alpha - 2C\mathbf{1}^\mathsf{T}\mathbf{H}_{t-1}^\mathsf{T}\mathbf{H}_{t-1}\alpha + C^2h_t^2 \qquad (17)$$
$$\leq C^2R^2 + 2C^2R^2 + C^2R^2 \qquad (18)$$
$$= 4C^2R^2 \qquad (19)$$

Applying Lemma B.3, we have

$$\max_{0 \leq \lambda \leq 1}[D(\alpha + \tilde{\alpha}(\lambda), \beta) - D(\alpha, \beta)] \geq \min\left\{\frac{\mu}{2}, \frac{\mu^2}{4C^2R^2}\right\} \qquad (20)$$

We update the $\alpha$ using the line search above and then optimize $\beta$ assuming $\alpha$ fixed. The dual problem B.2 is a quadratic function with a diagonal quadratic matrix. Thus there is no interaction between each coordinate of $\beta$, and they can be optimized independently.

The optimal solution is

$$\forall j \in P, \quad (\beta^*)_j = \max(0, -(\mathbf{H}\alpha)_j) \qquad (21)$$

with an increase in the objective

$$\frac{1}{2}(\beta)_j^2 + (\mathbf{H}\alpha)_j(\beta)_j, \text{ if } (\beta^*)_j = 0; \qquad (22)$$
$$\frac{1}{2}((\beta^*)_j - (\beta)_j)^2, \text{ if } (\beta^*)_j = -(\mathbf{H}\alpha)_j; \qquad (23)$$

It is important to check that this solution ensures that $\mathbf{w} \geq \mathbf{0}$. In both cases, the component-wise update in $\beta$ gives the objective a non-negative increase. However, the increase can be zero when $(\beta)_j = 0$ or $(\mathbf{H}\alpha)_j \leq 0$, or equivalently, when the primal constraint $\mathbf{w}_j \geq 0$ is not activated.

In summary, adding the non-negative constraints will not widen the duality gap but will actually decrease the gap, yet the amount of reduction is not guaranteed, as is the case with $\alpha$.

The remainder of the reasoning is identical to [4]. The reasoning leads to the following theorem:

**Theorem B.4. Convergence of Algorithm 1** For any training dataset $\mathcal{D}$ and any $C > 0$, $0 < \varepsilon \leq 4R^2C$, $\rho > 0$, Algorithm 1 terminates after at most

$$\left\lceil\log_2\frac{\Delta(\rho)}{4R^2C}\right\rceil + \left\lceil\frac{16R^2C}{\varepsilon}\right\rceil \qquad (24)$$

iterations.

We have enforced submodularity for the loss augmented inference, thus it can be computed optimally using the BK algorithm [1] with worst case complexity $O(n^2 m|\mathcal{C}|)$ or the standard push-relabel based max-flow algorithm [2] with worst case complexity $O(n^2\sqrt{m})$ [1]. Here $n$ and $m$ denote the number of nodes and edges in the graph. $|\mathcal{C}|$ is the size of the minimal cut.

In each iteration of Algorithm 1, the loss augmented inference is called exactly $n$ times, with $n$ being the size of the dataset. Putting everything together, we have the proof for Theorem 5.2, *i.e.*, polynomial time termination of Algorithm 1.

## C. Proof for Generalization to Higher Order Potentials

Our algorithm can be generalized to higher order potentials using the reduction described in [3]. Let

$$S_1 = \sum_{i=1}^{d} y_i, \quad S_2 = \sum_{i=1}^{d-1}\sum_{j=i+1}^{d} y_i y_j = \frac{S_1(S_1-1)}{2} \quad (25)$$

The two ways of reduction are proposed based on the sign of the coefficient $a$:

if $a < 0$,

$$ay_1...y_d = \min_{z\in\{0,1\}} az(S_1-d+1) \quad (26)$$

if $a > 0$,

$$ay_1...y_d = a \min_{z_1,...,z_{n_d}\in\{0,1\}} \sum_{i=1}^{n_d} z_i[c_{i,d}(-S_1+2i)-1] + aS_2 \quad (27)$$

where $n_d$ and $c_{i,d}$ are some positive constants.

In our case, $a = -\mathbf{w}_d \cdot \phi(u_1,...,u_d)$. To enforce submodularity, we want all coefficients of the pairwise terms to be non-positive. It can be verified that if $a < 0$, this condition is satisfied. If $a < 0$, we have, after reduction, the term $aS_2$, which contains positive coefficients. Thus, we need to impose similar assumptions and restrictions that all high order features are non-negative and the learned higher order potential be non-negative. Applying this reduction, our algorithm is able to learn the parameters for high order potentials exactly in polynomial time.

## References

[1] Y. Boykov and V. Kolmogorov. An experimental comparison of min-cut/max-flow algorithms for energy minimization in vision. *Pattern Analysis and Machine Intelligence, IEEE Transactions on*, 26(9):1124–1137, 2004. 3

[2] A. V. Goldberg and R. E. Tarjan. A new approach to the maximum-flow problem. *Journal of the ACM (JACM)*, 35(4):921–940, 1988. 3

[3] H. Ishikawa. Transformation of general binary mrf minimization to the first-order case. *Pattern Analysis and Machine Intelligence, IEEE Transactions on*, 33(6):1234–1249, 2011. 3

[4] T. Joachims, T. Finley, and C.-N. J. Yu. Cutting-plane training of structural svms. *Machine Learning*, 77(1):27–59, 2009. 1, 2

[5] C. H. Teo, A. Smola, S. Vishwanathan, and Q. V. Le. A scalable modular convex solver for regularized risk minimization. In *Proceedings of the 13th ACM SIGKDD international conference on Knowledge discovery and data mining*, pages 727–736. ACM, 2007. 1

---

[1]Although the BK algorithm has a worse theoretical complexity, it was shown in [1] to be more efficient for computer vision problems in practice.



\end{document}